\documentclass{article}

\usepackage{setspace}
\usepackage[left=1in,right=1in,top=1in,bottom=1in]{geometry}
\usepackage{amsthm, amsmath, mathtools, mathrsfs}
\usepackage{bm}
\usepackage[utf8]{inputenc} 
\usepackage[T1]{fontenc}    
\usepackage[colorlinks=true,allcolors=blue]{hyperref}       
\usepackage{url}            
\usepackage{booktabs}       
\usepackage{amsfonts}       
\usepackage{nicefrac}       
\usepackage{microtype}      
\usepackage{lipsum}
\usepackage{relsize}
\usepackage{algorithm}
\usepackage{enumitem}
\usepackage{algorithmic}
\usepackage{natbib}
\usepackage{color}

\usepackage{sectsty}

\allsectionsfont{\mdseries\scshape}

\newcommand{\tr}{\mathrm{tr}}
\newcommand{\kl}{D_{\mathrm{KL}}}
\newcommand{\dd}{\mathrm{d}}
\newcommand{\ito}{\text{It\^{o}}}
\newcommand{\strat}{\text{Strat}}

\renewcommand{\v}[1]{\boldsymbol{#1}}

\newtheorem{theorem}{Theorem}
\newtheorem*{theorem*}{Theorem}

\theoremstyle{definition}

\title{\textsc{Stochastic Normalizing Flows}}
\author{
Liam Hodgkinson%
\thanks{Department of Statistics, University of California at Berkeley, USA, and International Computer Science Institute, Berkeley, CA, USA. Email: \texttt{liam.hodgkinson@berkeley.edu}}
\and 
Chris van der Heide%
\thanks{School of Mathematics and Physics, University of Queensland, Australia. Email:  \tt{chris.vdh@gmail.com}}
\and 
Fred Roosta%
\thanks{School of Mathematics and Physics, University of Queensland, Australia, and International Computer Science Institute, Berkeley, CA, USA. Email:  \tt{fred.roosta@uq.edu.au}}
\and 
Michael W. Mahoney%
\thanks{Department of Statistics, University of California at Berkeley, USA, and International Computer Science Institute, Berkeley, CA, USA. Email: \texttt{mmahoney@stat.berkeley.edu}}
}

\begin{document}
\maketitle
\begin{abstract}
We introduce stochastic normalizing flows, an extension of continuous normalizing flows for maximum likelihood estimation and variational inference (VI) using stochastic differential equations (SDEs). Using the theory of rough paths, the underlying Brownian motion is treated as a latent variable and approximated, enabling efficient training of neural SDEs as random neural ordinary differential equations. These SDEs can be used for constructing efficient Markov chains to sample from the underlying distribution of a given dataset. Furthermore, by considering families of targeted SDEs with prescribed stationary distribution, we can apply VI to the optimization of hyperparameters in stochastic MCMC.
\end{abstract}

\section{Introduction}

\emph{Normalizing flows} \citep{rezende2015variational} are probabilistic models constructed as a sequence of successive transformations applied to some initial distribution. A key strength of normalizing flows is their expressive power as generative models, while enjoying an explicitly computable form of the likelihood function evaluated on the transformed space. This makes them especially well-equipped for variational inference (VI). Neural networks are often used as inspiration for finding effective transformations \citep{dinh2014nice,berg2018sylvester}. 

\emph{Continuous normalizing flows} were later developed in \cite{chen2018neural} as a means to perform maximum likelihood estimation and VI for large-scale probabilistic models derived from ordinary differential equations (ODEs). The framework stems from the computation of the evolving density of an ODE with random initial value, as the solution to another ODE.
The jump to continuous-time dynamics affords a few computational benefits over its discrete-time counterpart, namely the presence of a trace in place of a determinant in the evolution formulae for the density, as well as the adjoint method for memory-efficient backpropagation. Motivated by deep learning, a family of ODEs, called \emph{neural ordinary differential equations} were constructed, whose Euler discretizations resembled layer-wise transformations of residual neural networks.
Further algorithmic improvements to the framework were presented by \citet{grathwohl2018ffjord}, enabling virtually arbitrary choices of parameterized classes of ODEs. 
Doing all this involves some technical subtlety, and effective neural ODE architectures remain the subject of ongoing research --- see for example \citep{dupont2019augmented,gholami2019anode,zhang2019anodev2}. 

There has also been recent interest in extending these frameworks to a stochastic scenario, that is, training probabilistic models derived from \emph{stochastic} differential equations (SDEs). 
For physical models, where the evolution of a dynamical system is no longer deterministic, or microscopic fluctuations are dependent on components changing too rapidly to quantify, an SDE can be more appropriate. Stochastic extensions of neural ODEs have been considered in \citep{tzen2019neural, liu2019neural, jia2019neural,peluchetti2019infinitely} as limits of deep latent Gaussian models, where they have been suggested to show increased robustness to noisy / adversarial data. Furthermore, unlike deterministic flows, there is a foolproof recipe for constructing a family of SDEs that are ergodic with respect to some target distribution \citep{ma2015complete}. This particular property guarantees the convergence of the solution of an SDE to a prescribed target distribution. Such SDEs are prime candidates for the construction of stochastic MCMC algorithms, by generating sample paths via approximate stochastic integration methods.

However, developing an analogue of the continuous normalizing flows framework for flows constructed from SDEs---in particular, one that comes with simple and rigorous mathematical theory and that does not rely on ad hoc or problem-specific assumptions---is far from trivial. A common approach for conducting VI with SDEs is to rely on Girsanov's theorem. This allows one to estimate the Kullback-Leibler divergence between densities of solutions to two SDEs (for the prior and posterior distributions) with differing drift coefficients \citep{beskos2006exact, tzen2019neural}. Following this approach, \citet{li2020scalable} developed a stochastic adjoint method which scales well to high dimensions, and enables SDEs as latent models in variational autoencoders. Theoretical justification of the method proved challenging, as stochastic calculus is ill-suited for analyzing backward (approximate) solutions to SDEs. Notable deficiencies with these previous approaches include difficulties with non-diagonal diffusion, incompatibility with higher-order adaptive SDE solvers, and a complex means of reconstructing Brownian motion paths from random number generator seeds. Furthermore, the method cannot be justifiably combined with existing approaches of density estimation for SDEs (see \citet{hurn2007seeing}).

On the other hand, recent efforts have made significant strides in applying variational and MCMC methods for idealized Bayesian computation. One of the most significant contributions in this direction is \citet{salimans2015markov}, who performed VI with respect to distributions formed from steps of a reversible Markov chain. For example, the setting of Hamiltonian Monte Carlo was examined in \citet{wolf2016variational}. More recently, \citet{liu2016two} considered optimizing step size in stochastic gradient Langevin dynamics using methods derived from kernelized Stein discrepancy. Langevin flows \citep{rezende2015variational} have been discussed as a potential VI framework that takes inspiration from the SDEs underlying stochastic MCMC \citep{ma2015complete}. Once again, implementation of Langevin flows relies on the approximation of the log-likelihood for a general class of SDEs. 

\subsection*{Contributions}
We provide a general theoretical framework (which we refer to as \emph{stochastic normalizing flows}) for approximating generative models constructed from SDEs using continuous normalizing flows. These approximations can then be trained using existing techniques. By this process, we find that theoretical and practical developments concerning continuous normalizing flows and neural ODEs extend readily to the stochastic setting, without the need of an independent framework. 
The key theoretical enabler underlying our strong results and simple analysis is the \emph{theory of rough paths} \citep{friz2014course}, an alternative stochastic calculus that enables approximation and pathwise treatment of SDEs. Our~approach
\begin{enumerate}
\item enables (i) density estimation, (ii) maximum likelihood estimation, and (iii) variational approximations beyond autoencoders, for \emph{arbitrary} SDE models; and
\item is easily implemented using any general continuous normalizing flows implementation, such as that of~\citet{grathwohl2018ffjord}.
\end{enumerate}
Our framework recovers the stochastic adjoint method of \citet{li2020scalable}, but our approach is sufficiently flexible to overcome its deficiencies.
Moreover, using our approach, any existing neural ODE framework (such as \citet{zhang2019anodev2}) can be extended to SDEs, simply by the addition of a few extra terms. 

Following a review of background material in \S\ref{sec:Review}, the stochastic normalizing flows framework is introduced and discussed in \S\ref{sec:Stochastic},
with our main approximation result presented in Theorem \ref{thm:RoughNormFlow}. Some numerical investigations are conducted in \S\ref{sec:Numerics}, including an application to hyperparameter optimization in stochastic MCMC.

\section{Background Review}
\label{sec:Review}

\subsection{Continuous Normalizing Flows}
We shall begin by reviewing the continuous normalizing flow framework for training ODE models, as our development of random and stochastic normalizing flows will build upon it. 
Consider a parameterized class of models $\{Z_\theta\}_{\theta \in \mathbb{R}^m}$ of the following form: for $f:\mathbb{R}^d \times [0,T] \times \mathbb{R}^m \to \mathbb{R}^d$, let $Z = Z_\theta \in \mathbb{R}^d$ satisfy the ODE with random initial condition (often called a \emph{random ordinary differential equation})
\begin{equation}
\label{eq:ODEModel}
\frac{\dd}{\dd t} Z(t) = f(Z(t), t, \theta),\quad 
Z(0) \sim p_0(\theta).
\end{equation}
In a general machine learning context, one might choose $f$ such that the Euler discretization of (\ref{eq:ODEModel}) resembles layer-wise updates of a residual neural network \citep{lu2017beyond,chen2018neural}, or one may parameterize $f$ as a neural network itself \citep{grathwohl2018ffjord}.
The resulting ODEs constitute the class of so-called \emph{neural ordinary differential equations}.
The following theorem is a consequence of the Liouville equation (equivalently, Fokker-Planck equation) applied to the
solution $Z(t)$ of the random ODE (\ref{eq:ODEModel}), and it yields an ODE for the log density of $Z(t)$ evaluated at $Z(t)$.

\begin{theorem}[\citet{chen2018neural}]
\label{thm:ContNormFlow}
Suppose that $Z(t)$ satisfies (\ref{eq:ODEModel}). The distribution of $Z(t)$ is absolutely continuous with respect to Lebesgue measure, with probability density $p_t$ satisfying
\begin{equation}
\label{eq:ContNormFlow}
\frac{\dd}{\dd t} \log p_t(Z(t)) = -\nabla_z \cdot f(Z(t), t, \theta) 
\end{equation}
\end{theorem}

Naively computing the divergence in (\ref{eq:ContNormFlow}) with automatic differentiation is of quadratic complexity in the dimension $d$. As pointed out by \citet{grathwohl2018ffjord}, this can be improved to linear complexity using a trace estimator \citep{roosta2015improved}:
\begin{equation}
\label{eq:TraceEst}
\nabla_z \cdot f(z) = \tr\left(\frac{\partial f}{\partial z}\right) 
\approx \frac1n \sum_{k=1}^{n} \epsilon_k^\top \frac{\partial f}{\partial z} \epsilon_k,
\end{equation}
where each $\epsilon_k$ is an independent and identically distributed copy of a random vector $\epsilon \in \mathbb{R}^d$ with zero mean and $\mathbb{E}[\epsilon \epsilon^\top] = I$. Common choices for $\epsilon_k$ include standard normal and Rademacher random vectors.

\subsection{The Adjoint Method}
Training continuous normalizing flows often involves minimizing a scalar loss function involving $Z$ and/or the log-density computed via Theorem \ref{thm:ContNormFlow} with respect to the parameters $\theta$. For this, we require gradients of $Z(t)$ with respect to $\theta$ for $t \in [0,T]$. The most obvious approach is to directly backpropagate through a numerical integration scheme such as in \citet{ryder2018black}, but this does not scale well in $T$.
The superior alternative is the \emph{adjoint method}, which computes derivatives of a scalar loss function by solving another differential equation in reverse time. Letting $L$ denote a scalar loss depending on $Z(T)$, the \emph{adjoint} given by $a(t) = \frac{\partial L}{\partial Z(t)}$, as well as the gradient of $L$ in $\theta$, satisfy \citep[\S12]{pontryagin2018mathematical}
\begin{subequations}
\label{eq:AdjointMethod}
\begin{align}
\frac{\dd}{\dd t} a(t) &= -\nabla_z f(Z(t),t,\theta) a(t), \\
\nabla_\theta L &= \int_0^T \nabla_\theta f(Z(t), t, \theta)a(t) \dd t.
\end{align}
\end{subequations}
Together with (\ref{eq:ODEModel}), the equations (\ref{eq:AdjointMethod}) are solved in reverse time, starting from the terminal values $Z(T)$ and $\nabla L(Z(T))$. By augmenting $Z(t)$ together with (\ref{eq:ContNormFlow}), this method also allows for loss functions depending on~$p_T(Z(T))$. 

Solving (\ref{eq:ODEModel}), (\ref{eq:ContNormFlow}), and (\ref{eq:AdjointMethod}) can be achieved using off-the-shelf numerical integrators. Adaptive solvers prove particularly effective, although, as pointed out in \citet{gholami2019anode}, the backward solve (\ref{eq:AdjointMethod}) can often run into stability issues, suggesting a Rosenbrock or other implicit approach \citep{hairer1996solving}. We point out that the same is also true in the stochastic setting; see \citet{hodgkinson2019implicit}, for example. For further implementation details concerning continuous normalizing flows, we refer to \citet{grathwohl2018ffjord}. 

\subsection{Rough Path Theory}

The theory of rough paths was first introduced in \citep{lyons1998differential} to provide a supporting pathwise theory for SDEs. 
It has since flourished into a coherent 
pathwise alternative to stochastic calculus, facilitating direct
stochastic generalizations of results from the theory of ODEs --- we refer to \citet{friz2014course} for a gentle introduction, and \citet{friz2010multidimensional} for a thorough treatment of the topic. 
Suppose that we would like to prescribe meaning to the infinitesimal limit of the sequence of iterates
\begin{equation}
\label{eq:RDELimit}
Z_{t+h} = Z_t + f(Z_t)(X_{t+h} - X_t),\quad \mbox{as }h \to 0^+.
\end{equation}
In the case of SDEs, $X_t$ is a sample path of Brownian motion, so that each $X_{t+h} - X_t$ is a realization of a normal random vector with zero mean and covariance $h I$.
Unfortunately, a strong limit of (\ref{eq:RDELimit}) fails to exist if $X_t$ is too ``rough''.
In particular, suppose that $X_t$ is $\alpha$-H\"{o}lder continuous for $\alpha \in (0,1)$, that is, there exists some $C > 0$ such that $\|X_s - X_t\| \leq C|s - t|^{\alpha}$ for any $s,t \geq 0$. Since the limit (\ref{eq:RDELimit}) is only well-defined if $\alpha \geq 1/2$ \citep{young1936inequality}, a function on $[0,T]$ is \emph{rough} if it is H\"{o}lder-continuous only for $\alpha < 1/2$. Sample paths of Brownian motion constitute rough paths under this definition. The problem is that the discretization (\ref{eq:RDELimit}) invokes the \emph{zeroth-order} approximation $f(Z_{t+s}) \approx f(Z_t)$ for $0 \leq s \leq h$, which proves too poor. By instead taking a \emph{first-order} approximation
\begin{align*}
f(Z_{t+s}) &\approx f(Z_t) + \nabla_z f(Z_t)(Z_{t+s}-Z_t) \\&\approx f(Z_t) + \nabla_z f(Z_t)f(Z_t)(X_{t+s}-X_t),
\end{align*}
we arrive at the \emph{Davie scheme} \citep{davie2008differential}
\begin{equation}
\label{eq:DavieApprox}
Z_{t+h} = Z_t + f(Z_t) (X_{t+h} - X_t) + \nabla_z f(Z_t)f(Z_t)\mathbb{X}_{t,t+h},
\end{equation}
where $\mathbb{X}_{s,t}$ represents the ``integral'' $\int_s^t X_r \dd X_r^\top$. Once again, we cannot uniquely define $\mathbb{X}$ from the path $X$ itself, so instead we prescribe it. In fact, each choice of $\mathbb{X}$ satisfying Chen's relations
\[
\mathbb{X}_{s,t} - \mathbb{X}_{s,u} - \mathbb{X}_{u,t} = (X_s - X_u)(X_t - X_u)^\top,
\]
for any $s,u,t \geq 0$, will reveal a \emph{different} limit for (\ref{eq:DavieApprox}) as $h \to 0^+$, provided $\alpha \geq 1/3$ (for smaller $\alpha$, higher-order approximations are necessary). The pair $\boldsymbol{X} = (X,\mathbb{X})$ is referred to as a \emph{rough path}, and the limit of (\ref{eq:DavieApprox}) as $h \to 0^+$ is the solution to the \emph{rough differential equation} (RDE) 
\begin{equation}
\label{eq:RoughDiffEq}
\dd \boldsymbol{Z}_t = f(Z_t) \dd \boldsymbol{X}_t.
\end{equation}
H\"{o}lder continuity is critical to rough path theory --- in the sequel, we equip the space of $\alpha$-H\"{o}lder functions with the $\alpha$-H\"{o}lder norm, defined by
\[
\|X\|_\alpha \coloneqq \sup_{t \in [0,T]}\|X_t\| + \sup_{\substack{s,t\in [0,T]\\s \neq t}}\frac{\|X_t - X_s\|}{|t - s|^{\alpha}}.
\]
This definition extends to the iterated integral $\mathbb{X}$ by replacing $X_t$ and $X_t - X_s$ with $\mathbb{X}_{0,t}$ and $\mathbb{X}_{s,t}$, respectively. 

It is useful to identify a calculus which satisfies the usual chain and product rules. This occurs when the rough path $\boldsymbol{X}$ is \emph{geometric}, that is,
\begin{equation}
\label{eq:GeoRoughPath}
\mathbb{X}_{s,t} - \mathbb{X}_{t,s} = \tfrac12 (X_t - X_s) (X_t - X_s)^\top,\quad \forall s,t \geq 0.
\end{equation}
Every continuous and piecewise differentiable function $X$ is canonically lifted to a geometric rough path by taking $\mathbb{X}_{s,t} = \int_s^t X_r \frac{\dd}{\dd r} X_r^\top \dd r$, where the derivative is interpreted in the weak sense. In these cases, (\ref{eq:RoughDiffEq}) equates to the ODE $\frac{\dd}{\dd t}Z_t = f(Z_t) \frac{\dd}{\dd t} X_t$.

Geometric rough paths have two key properties of interest:
\begin{enumerate}[label=\Roman*.]
    \item The canonical lifts of any sequence of smooth approximations $X^{(n)}$ which converge to $X$ as $n \to \infty$ in the $\alpha$-H\"{o}lder norm, also converge in the $\alpha$-H\"{o}lder rough path metric 
    \[
    \varrho_{\alpha}((X,\mathbb{X}),(Y,\mathbb{Y})) = \|X - Y\|_{\alpha} + \|\mathbb{X} - \mathbb{Y}\|_{2\alpha},
    \]
    to a geometric rough path $(X,\mathbb{X})$. Conversely, any geometric rough path can be approximated by some sequence of smooth paths \citep[Proposition 2.5]{friz2014course}.
    \item The reverse-time process $\tilde{Z}_t = Z_{T-t}$ of a solution $Z_t$ to any rough differential equation (\ref{eq:RoughDiffEq}) with Lipschitz $f$, itself satisfies the reversed rough differential equation $\dd \tilde{\bm{Z}}_t = -f(T-t,\tilde{Z}_t)\dd \bm{X}_{T-t}$ \emph{if and only if} $\bm{X}$ is~geometric.
\end{enumerate}
By property I, any solution to RDEs driven by a geometric rough path can be approximated by solutions to ODEs. Property II, which follows readily from the definition (\ref{eq:GeoRoughPath}) in the limit (\ref{eq:DavieApprox}), enables the adjoint method for rough differential equations driven by a geometric rough path.

\section{Stochastic Normalizing Flows}
\label{sec:Stochastic}

Let $Z_t$ satisfy the It\^{o} SDE
\begin{equation}
\label{eq:MainSDE}
\dd Z_t = \mu_t(Z_t, \theta) \dd t + \sigma_t(Z_t, \theta) \dd B_t, \quad Z_0 \sim p_0(\theta),
\end{equation}
where $B_t$ is an $m$-dimensional Brownian motion, and $\mu_t:\mathbb{R}^d \to \mathbb{R}^d$,  $\sigma_t:\mathbb{R}^d \to \mathbb{R}^{d \times m}$ are the drift, and diffusion coefficients, respectively. Analogous to neural ODEs, neural SDEs choose $\mu_t$ to resemble a single layer of a neural network \citep{tzen2019neural}. The dropout-inspired construction of \citet{liu2019neural} suggests taking $\sigma_t \propto \mathrm{diag}(\mu_t)$. Alternatively, one can parameterize both $\mu_t$ and $\sigma_t$ by multi-layer neural networks.

The reliance of stochastic calculus on non-anticipating processes as well as the lack of continuity for solution maps of It\^{o} SDEs necessitates complicated and delicate arguments for extending each piece of the continuous normalizing flow framework from \S\ref{sec:Review} to SDEs. We bypass the intricacies of existing theoretical treatments of neural SDEs by an approximation argument: for a smooth approximation $\tilde{B}_t$ of Brownian motion $B_t$, we estimate solutions of an SDE by a random ODE involving $\tilde{B}_t$. 
One must take great care with such approximations. For example, geometric Brownian motion, that is, the solution to $\dd Z_t = \sigma Z_t \dd B_t$, has the explicit expression $Z_t = Z_0 \exp(-\frac{\sigma^2}{2} t + \sigma B_t)$, which is not well-approximated by the solution $\tilde{Z}_t = Z_0 \exp(\sigma \tilde{B}_t)$ to $\frac{\dd}{\dd t} \tilde{Z}_t = \sigma \tilde{Z}_t \frac{\dd \tilde{B}_t}{\dd t}$. Theoretical verification of this approach is challenging using traditional stochastic calculus due to the irregularity of solution maps. Instead, we rely on rough path theory --- particularly properties I and II of geometric rough paths. 

In the rough path framework, one can reconstruct the It\^{o} stochastic calculus via the rough path $\bm{B}^{\ito} = (B, \mathbb{B}^{\ito})$, where $\mathbb{B}^{\ito}_{s,t} = B_t (B_t - B_s)^\top - \frac{t-s}2 I$. Indeed, by \citet[Theorem 9.1]{friz2014course}, letting $\bm{B}^{\ito}(\omega)$ denote a realization of the It\^{o} Brownian motion rough path, the solution to the rough differential equation
\begin{equation}
\label{eq:ItoRDE}
\dd \bm{Z}_t = \mu_t(Z_t, \theta) \dd t + \sigma_t(Z_t, \theta) \dd \bm{B}^{\ito}_t(\omega)
\end{equation}
is a realization of the strong solution to (\ref{eq:MainSDE}). Likewise, the Davie scheme (\ref{eq:DavieApprox}) corresponds to the Milstein integrator for SDEs \citep[\S10.3]{kloeden2013numerical}. 

Unfortunately, $\bm{B}^{\ito}(\omega)$ is not a geometric rough path, and so Theorem \ref{thm:RoughNormFlow} cannot be directly applied. Instead, we shall proceed according to the following steps:
\begin{enumerate}[label=(\roman*)]
    \item Convert the It\^{o} SDE to a Stratonovich SDE (\S\ref{sec:Strat}).
    \item Interpret the Stratonovich SDE pathwise as an RDE driven by a geometric rough path $\bm{B}^{\strat}$ (\ref{eq:StratRDE}).
    \item Approximate the pathwise Stratonovich RDE by a random ODE (\S\ref{sec:WongZakai}).
    \item Train the random ODE as a continuous normalizing flow with added latent variables (\S\ref{sec:RandCNF}).
\end{enumerate}
Consequently, the RDE (\ref{eq:ItoRDE}) is estimated by the ODE
$\frac{\dd Z_t(\omega)}{\dd t} = F_\omega(Z_t(\omega),t,\theta)$
where
\[
F_\omega(z, t, \theta) = \underset{\text{Stratonovich drift}}{\underbrace{\tilde{\mu}_t(z, \theta)}} + \;\;\sigma_t(z,\theta)\underset{\text{approximation}}{\underbrace{\frac{\dd B_t(\omega)}{\dd t}}}.
\]

\subsection{Stratonovich calculus}
\label{sec:Strat}
The unique geometric rough path formed from Brownian motion yields the Stratonovich calculus: $\mathbb{B}^{\strat}_{s,t} = B_t(B_t - B_s)^\top$.
A Stratonovich differential equation is commonly written in the form $\dd Z_t = \mu_t(Z_t) \dd t + \sigma_t(Z_t) \circ \dd B_t$, where $\circ$ denotes Stratonovich integration: for a process $Y_t$ adapted to the filtration generated by $B_t$,
\[
\int_s^t Y_t \circ \dd B_t = \lim_{|\mathcal{P}|\to 0}\sum_{k=1}^N \frac12(Y_{t_k} + Y_{t_{k-1}}) (B_{t_k} - B_{t_{k-1}}),
\]
where $\mathcal{P} = \{0 = t_0 < \cdots < t_N = T\}$ is a partition with mesh size $|\mathcal{P}| = \max_k |t_k - t_{k-1}|$, and the limit is in $L^2$. This is to be compared with It\^{o} integration which is defined instead by
\[
\int_s^t Y_t \,\dd B_t = \lim_{|\mathcal{P}|\to 0}\sum_{k=1}^N Y_{t_{k-1}} (B_{t_k} - B_{t_{k-1}}).
\]
Stratonovich differential equations were recognized in \citet{li2020scalable} to be the correct setting for extending the adjoint method to SDEs. However, the adherence to classical stochastic calculus, which relies on adaptedness, somewhat complicates the argument. In our setting, the advantages of Stratonovich differential equations are clear. Because Stratonovich differential equations can be arbitrarily well-approximated by random ODEs, \emph{all} methods of training continuous normalizing flows extend to them, including the adjoint method. Any It\^{o} SDE can be converted into a Stratonovich SDE by adjusting the drift
\citep[p. 123]{evans2012introduction}, a fact readily seen by comparing limits of (\ref{eq:DavieApprox}) with $\mathbb{B}^{\ito}$ and $\mathbb{B}^{\strat}$.
The following formula is particularly amenable to implementation with automatic differentiation: the It\^{o} SDE $\dd Z_t = \mu_t(Z_t) \dd t + \sigma_t(Z_t) \dd B_t$ is equivalent to the Stratonovich SDE $\dd Z_t = \tilde{\mu}_t(Z_t) \dd t + \sigma_t(Z_t) \circ \dd B_t$ provided that for each $i=1,\dots,d$,
\begin{equation}
\label{eq:ItoCorrection}
\tilde{\mu}_t^i(x) = \mu_t^i(x) - \tfrac12\nabla_x \cdot (\sigma_t(x)\sigma_t^\top(x^{\ast}))_i,
\end{equation}
where $x^\ast$ is an independent copy of $x$, and the subscript denotes the $i$-th row.
Once again, we can make use of the trace estimator (\ref{eq:TraceEst}) to increase performance in higher dimensions. In the rough path theory, Stratonovich SDEs are interpreted pathwise according to the RDE
\begin{equation}
\label{eq:StratRDE}
\dd \bm{Z}_t = \tilde{\mu}_t(Z_t, \theta) \dd t + \sigma_t(Z_t,\theta) \dd \bm{B}_t^{\strat}(\omega),
\end{equation}
which is equivalent to (\ref{eq:ItoRDE}).

\subsection{Wong--Zakai approximations}
\label{sec:WongZakai}

A random ODE $\frac{\dd}{\dd t} Z_t^{(n)} = \mu_t(Z_t^{(n)}) + \sigma_t(Z_t^{(n)}) \frac{\dd B_t^{(n)}}{\dd t}$ estimating a Stratonovich SDE $\dd Z_t = \mu_t(Z_t) \dd t + \sigma_t(Z_t) \circ \dd B_t$ is commonly referred to as a \emph{Wong--Zakai approximation} \citep{twardowska1996wong}, after the authors of the seminal paper \citep{wong1965convergence}, who first illustrated this concept for one-dimensional Brownian motion. We shall consider two types of Wong--Zakai approximation: a Karhunen-Lo\`{e}ve expansion, and a piecewise linear function. These approximations are compared in Figure \ref{fig:WienerApprox}. In practice, we have found that the Karhunen-Lo\`{e}ve expansion with $4 \leq n \leq 10$ terms works well for training, while the piecewise linear approximation is preferable for testing. 
\begin{figure}
    \centering
    \includegraphics[width=\columnwidth]{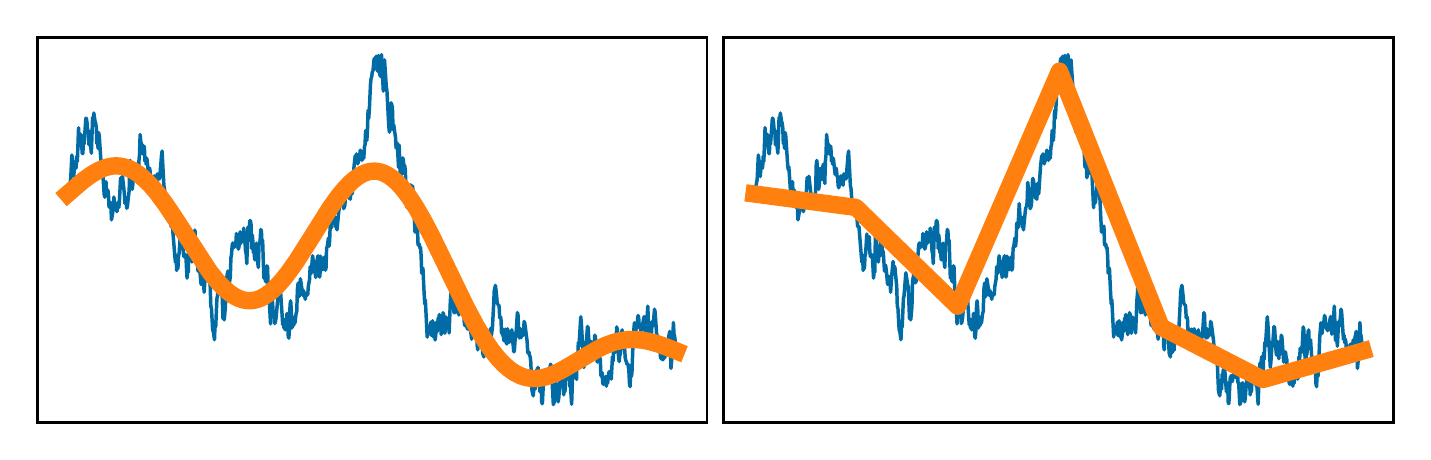}
    \caption{Karhunen-Loeve (left) and piecewise linear (right) approximations of a Brownian motion sample path with $n = 6$ and $\Delta t = \frac{1}{6}$ respectively.}
    \label{fig:WienerApprox}
\end{figure}

\subsubsection{Piecewise linear}
Easily the most common approximation of Brownian motion involves exact simulation on a discrete set of times $\{0,t_1,t_2,\dots,t_n\}$, followed by linear interpolation. More precisely, letting $\Delta t_k = t_{k+1}-t_k$, for each $k=0,\dots,n-1$, we let
\[
B_{t_{k+1}}^{(n)} = B_{t_k}^{(n)} + \sqrt{\Delta t_k}\omega_k,\quad \omega_k \sim \mathcal{N}(0,1),
\]
and consider the approximation
\[
B_t^{(n)} = B_{t_k}^{(n)} + \frac{t - t_k}{t_{k+1}-t_k} (B_{t_{k+1}}^{(n)} - B_{t_k}^{(n)}),\quad t \in [t_k,t_{k+1}].
\]
Integrating the resulting Wong--Zakai approximation using Euler's method on the same set of time points is equivalent to performing the Euler--Maruyama method for solving the Stratonovich SDE. By \citet[Theorem 15.45]{friz2010multidimensional}, as the mesh size $\max_k \Delta t_k \to 0$, the piecewise linear approximation converges almost surely to Brownian motion in the $\alpha$-H\"{o}lder norm for any $\alpha < 1/2$. 

\subsubsection{Karhunen-Lo\`{e}ve expansion}

For any zero-mean Gaussian process $X_t$ on $\mathbb{R}^d$ with $t\in [0,T]$, the covariance function $K(s,t) = \mathbb{E}[X_s X_t^\top]$ is a positive-definite kernel. If $K$ is also continuous, Mercer's theorem guarantees the existence of an orthonormal basis on $L^2([0,T],\mathbb{R}^d)$ of eigenfunctions $\{e_k\}_{k=1}^{\infty}$ with corresponding positive eigenvalues $\{\lambda_k\}_{k=1}^{\infty}$ such that $K(s,t)=\sum_{j=1}^{\infty} \lambda_j e_j(s) e_j(t)$. The process $X_t$ can be expanded in terms of these eigenfunctions as
\[
X_t = \sum_{k=1}^{\infty} \sqrt{\lambda_k} \omega_k e_k(t), \quad \omega_k \sim \mathcal{N}(0,1),
\]
where each $\omega_k$ is independent. This is called the \emph{Karhunen-Lo\`{e}ve expansion} of $X$. Truncating the series after $n$ terms yields the $n$-th order Karhunen-Lo\`{e}ve approximation, and has the smallest mean squared error over all expansions with $n$ orthogonal basis elements. Recalling that we are primarily interested in the endpoints of the solution, instead of expanding Brownian motion itself, we consider an approximation $B_t^{(n)}$ derived from the Karhunen-Lo\`{e}ve expansion of the Brownian bridge $B_t - B_T \frac{t}{T}$:
\[
B_t^{(n)} = \omega_0 \frac{t}{\sqrt{T}} + \sum_{k=1}^{n - 1} \omega_k \frac{\sqrt{2T} \sin(k\pi t  / T)}{k\pi},\quad n = 1,2,\dots.
\]
Using this approximation ensures that the terminal density for SDEs with constant drift and diffusion coefficients is computed exactly. By \citet[Theorem 15.51]{friz2010multidimensional}, $(B_t^{(n)})_{t \in [0,T]}$ converges almost surely as $n \to \infty$ to Brownian motion in the $\alpha$-H\"{o}lder norm for any $\alpha < 1/2$. Furthermore, since $B_t^{(n)}$ is smooth, Wong--Zakai approximations involving $\frac{\dd B_t^{(n)}}{\dd t}$ may be readily solved using adaptive ODE solvers. 

\subsection{Main result}

Using Wong--Zakai approximations, a Stratonovich SDE can be uniformly approximated in H\"{o}lder norm by random ODEs. In Theorem \ref{thm:RoughNormFlow}, we show that the log-densities and loss function gradients for these random ODEs also converge appropriately. More generally, geometric rough paths (including the Stratonovich paths (\ref{eq:StratRDE})) with random initial conditions can be approximately trained as random ODEs.

\begin{theorem}
\label{thm:RoughNormFlow}
Let $\bm{X} = (X, \mathbb{X})$ be an $\alpha$-H\"{o}lder geometric rough path, and $\{X^{(n)}\}_{n=1}^{\infty}$ a sequence of piecewise differentiable functions on $[0,T]$ that approximate $X$ under the $\beta$-H\"{o}lder norm for $\beta \in (\frac13, \frac12)$, that is, $\|X^{(n)} - X\|_\beta \to 0$ as $n \to \infty$. Let $\bm{Z},Z^1,Z^2,\dots$ be solutions to the differential equations
\begin{subequations}
\label{eq:RoughFlow}
\begin{align}
\dd \bm Z_t &= f(Z_t, t, \theta) \dd \bm{X}_t, &&& Z_0 &\sim p_0,\\ \frac{\dd Z_t^{(n)}}{\dd t} &= f(Z_t^{(n)}, t, \theta) \frac{\dd X_t^{(n)}}{\dd t} &&& Z_0^{(n)} &= Z_0,
\end{align}
\end{subequations}
where $f \in \mathcal{C}_b^4(\mathbb{R}^d \times [0,T] \times \mathbb{R}^m)$ and $p_0$ is a density on $\mathbb{R}^d$ such that $\log p_0$ is continuous. Let $p_t^{(n)}$ denote the probability density of $Z_t^{(n)}$ at time $t$, given by (\ref{eq:ContNormFlow}). The distribution of $Z_t$ is absolutely continuous with respect to Lebesgue measure with corresponding continuous density $p_t$ satisfying:
\begin{enumerate}
\item For any $x \in \mathbb{R}^d$, $\displaystyle{\sup_{t \in [0,T]}}|\log p_t^{(n)}(x) - \log p_t(x)| \to 0$ as $n \to \infty$.
\item The path $t \mapsto \log p_t(Z_t)$ is the unique solution to the rough differential equation
\begin{equation}
\label{eq:RoughNF}
\dd \log p_t(Z_t) = -\nabla_z \cdot (f(Z_t, t, \theta) \dd \bm X_t). 
\end{equation}
\item For any smooth loss function $L:\mathbb{R}^{d+1}\to\mathbb{R}$ and $t \geq 0$, as $n \to \infty$,
\begin{equation}
\label{eq:ApproxAdjoint}
\nabla_\theta L(Z_t^{(n)},\log p_t^{(n)}(Z_t^{(n)})) \to \nabla_\theta L(Z_t,\log p_t(Z_t)).
\end{equation}
\end{enumerate}
\end{theorem}
\begin{proof}[Proof of Theorem \ref{thm:RoughNormFlow}]
Recall that each $X^{(n)}$ can be lifted canonically to a rough path $\bm{X}^{(n)}$ such that $\rho_{\beta}(\bm{X}^{(n)},\bm{X}) \to 0$ as $n \to \infty$. For an arbitrary rough path $\bm{Y}$, we let $\Phi_t(\boldsymbol{Y}, \xi)$ and $\Psi_t(\bm{Y},\ell)$ denote the solution maps for the rough differential equations $\dd \bm{Z}_t = f(Z_t,t,\theta)\dd \bm{Y}_t$, $Z_0 = \xi$ and $\dd \bm{L}_t = -\nabla_z\cdot f(Z_t,t,\theta) \dd \bm{Y}_t$, $L_0 = \ell$, respectively. By \citet[Theorem 8.10]{friz2014course}, $\Phi_t(\boldsymbol{Y},\cdot)$ is a $\mathcal{C}^1$-diffeomorphism, and hence, for $Z_0 \sim p_0(\theta)$ and any $t \in [0,T]$, $Z_t = \Phi_t(\boldsymbol{X}, Z_0)$ is an absolutely continuous random variable, whose corresponding density we denote by $p_t$. In fact, denoting by $\Phi_{-t}(\boldsymbol{Y},\cdot)$ the inverse of $\Phi_t(\bm{Y},\cdot)$,
\begin{align}
\label{eq:DensityNCoM}
p_t^{(n)}(x) &= p_0(\Phi_{-t}(\bm{X}^{(n)},x))\left|\det \frac{\partial \Phi_{-t}(\bm{X}^{(n)},x)}{\partial x}\right|,\\
\label{eq:DensityCoM}
p_t(x) &= p_0(\Phi_{-t}(\bm{X},x))\left|\det \frac{\partial \Phi_{-t}(\bm{X},x)}{\partial x}\right|,
\end{align}
and so both $p_t^{(n)}$ and $p_t$ are continuous. Furthermore, by \citet[Theorem 8.5]{friz2014course}, for any $\frac13 < \gamma < \beta$, there exist constants $C_\gamma^{\Phi}$ and $C_\gamma^{\Psi}$ such that for any $\beta$-H\"{o}lder continuous rough paths $\boldsymbol{X}$, $\boldsymbol{Y}$ and $\xi, \tilde{\xi} \in \mathbb{R}^d$, $\ell, \tilde{\ell} \in \mathbb{R}_+$,
\begin{align}
\label{eq:PhiRegularity}
\|\Phi(\boldsymbol{X}, \xi) - \Phi(\boldsymbol{Y}, \tilde{\xi})\|_{\gamma} &\leq C_\gamma^\Phi (\|\xi - \tilde{\xi}\| + \varrho_{\beta}(\boldsymbol{X}, \boldsymbol{Y}))\\
\label{eq:PsiRegularity}
\|\Psi(\bm{X}, \ell) - \Psi(\bm{Y}, \tilde{\ell})\|_{\gamma} &\leq C_\gamma^\Psi (|\ell - \tilde{\ell}| + \varrho_{\beta}(\bm{X},\bm{Y})).
\end{align}
We deduce the following for any $t \in [0,T]$ and $x \in \mathbb{R}^d$: \textbf{(i)} $\|Z^{(n)} - Z\|_\gamma \to 0$ by (\ref{eq:PhiRegularity}); \textbf{(ii)} using (i) and continuity of $p_t$, $\log p_t(Z_t^{(n)}) \to \log p_t(Z_t)$; \textbf{(iii)} as a consequence of (\ref{eq:DensityNCoM}), (\ref{eq:DensityCoM}), (\ref{eq:PhiRegularity}), and \citet[Theorem 8.10]{friz2014course}, $p_t^{(n)}(x)\to p_t(x)$; \textbf{(iv)} combining (ii) and (iii), $\log p_t^{(n)}(Z_t^{(n)}) \to \log p_t(Z_t)$. Since $\Psi_t(\bm{X}^{(n)},\log p_0(Z_0)) = \log p_t^{(n)}(Z_t^{(n)})$ by Theorem \ref{thm:ContNormFlow}, (iv) and (\ref{eq:PsiRegularity}) imply $\log p_t(Z_t) = \Psi(\bm{X}, \log p_0(Z_0))$ and hence (\ref{eq:RoughNF}).
Let $x \in \mathbb{R}^d$ be arbitrary. To show that $\log p_t^{(n)}(x)$ converges uniformly in $t$, observe that
\[
\log p_t(x) = \Psi(\bm{X}, \log p_0(\Phi_{-t}(\bm{X},x))),
\]
and similarly for $\log p_t^{(n)}(x)$. Together with property II of geometric rough paths, inequality (\ref{eq:PhiRegularity}) with $\bm{Y} \equiv \bm{0}$ reveals that $\Phi_{-t}(\bm{X}^{(n)},x)$ and $\Phi_{-t}(\bm{X},x)$ are uniformly bounded in $t \in [0,T]$. Since $\log p_0$ is continuous, $\log p_0(\Phi_{-t}(\bm{X}^{(n)},x))$ converges to $\log p_0(\Phi_{-t}(\bm{X},x))$ uniformly in $t \in [0,T]$. 
Applying (\ref{eq:PsiRegularity}),
\begin{multline*}
\sup_{t\in[0,T]}|\log p_t^{(n)}(x) - \log p_t(x)| \leq C_{\gamma}^\Psi(\varrho_\beta(\bm{X}^{(n)}, \bm{X}) 
+ |\log p_0(\Phi_{-t}(\bm{X}^{(n)}, x)) - \log p_0(\Phi_{-t}(\bm{X}, x))| \to 0.
\end{multline*}
Finally, to prove (\ref{eq:ApproxAdjoint}), by \citet[Proposition 5.6]{friz2014course}, we can write $(\theta,L)$ as the solution to the rough differential equation
\begin{subequations}
\label{eq:LossSDE}
\begin{align}
&\dd \theta = 0 \\
&\dd L(Z_t,\log p_t(Z_t)) = \nabla_z L(Z_t, \log p_t(Z_t)) \cdot \dd Z_t + \nabla_\ell L(Z_t, \log p_t(Z_t)) \dd \log p_t(Z_t)
\end{align}
\end{subequations}
and similarly for $Z_t^{(n)}$ and $\log p_t^{(n)}(Z_t^{(n)})$, where $\nabla_z$ and $\nabla_\ell$ denote the gradients with respect to $Z_t$ and $\log p_t(Z_t)$, respectively. The derivative of $L$ with respect to $\theta$ is a derivative of (\ref{eq:LossSDE}) with respect to its initial condition, and hence (\ref{eq:ApproxAdjoint}) follows from \citet[Theorem 8.10]{friz2014course}.
\end{proof}

\subsection{Random continuous normalizing flows}
\label{sec:RandCNF}

By a conditioning argument, any random ODE, such as a Wong-Zakai approximation, may be treated as a continuous normalizing flow. Let $Z_t$ be the solution to a random ODE of the form
\begin{equation}
\label{eq:RandomODE}
\frac{\dd}{\dd t}Z_t = f(Z_t, \omega, t, \theta),
\end{equation}
where $\omega = (\omega_1,\dots,\omega_n) \sim q(\omega)$ is a random vector independent of $Z_t$, $t$, and $\theta$. The reduction of a random ODE to this form is in keeping with the reparameterization trick \citep{xu2019variance}. In particular, for the piecewise linear and Karhunen-Lo\`{e}ve approximations, each $\omega_i \sim \mathcal{N}(0,1)$. 
After conditioning on $\omega$, Theorem \ref{thm:ContNormFlow} applied to (\ref{eq:RandomODE}) provides a means of computing $\log p_t(Z_t\vert\omega)$, after sampling $Z_0 \sim p_0$. The density $p_t(Z_t)$ can be computed using a naive Monte Carlo estimator
\begin{equation}
\label{eq:DensityEstNaive}
p_t^\theta(Z_t) = \int p_t^\theta(Z_t \vert \omega)q(\omega) \dd \omega \approx \frac1N \sum_{i=1}^{n} p_t^\theta(Z_t \vert \omega_i),
\end{equation}
where the dependence on $\theta$ has been made explicit, and can be optimized over using the adjoint method. Analogously to \citet{chen2018neural,grathwohl2018ffjord}, the density of data $\v x$ may be estimated along a single sample path $B_t(\omega)$ (denoted $p(\v x\vert \omega)$) in the following way: letting $\Delta \log p_t^\omega = \log p_t(Z_t\vert \omega) - \log p(\v x \vert \omega)$, we see that $\Delta \log p_t^\omega$ also satisfies (\ref{eq:ContNormFlow}). By solving (\ref{eq:RandomODE}) and the corresponding (\ref{eq:ContNormFlow}) in reverse time from the initial conditions $Z_T = \v x$ and $\Delta \log p_T^\omega = 0$, we obtain $Z_0$ and $\Delta \log p_0^\omega$, and compute $\log p(\v x \vert \omega) = \log p_0(Z_0) - \Delta \log p_0^\omega$. This is shown in Algorithm \ref{alg:DensityEst}, which depends on an ODE solver \textsc{odesolve}, and yields a density estimation procedure for stochastic normalizing flows when paired with (\ref{eq:DensityEstNaive}). Note that by comparison to \citet[Algorithm 1]{grathwohl2018ffjord} which encompasses steps 6--12 of our Algorithm \ref{alg:DensityEst}, we see much of the density estimation procedure can be accomplished using an existing continuous normalizing flow implementation. In variational settings where $\log p_T(Z_T)$ is required, the same procedure applies, where $\mathbf{x}$ becomes $Z_T$ and is generated by the SDE as well.

\begin{algorithm}[htbp]
   \caption{Stochastic normalizing flows (density estimation; single path)}
   \begin{algorithmic}
   \STATE {\bfseries Input:} drift function $\mu$, diffusion function $\sigma$, an initial distribution $p_0$, final time $T$, minibatch of samples $\v x$, sample path $\tilde{B}_t(\omega)$ of Brownian motion approximation.
   \STATE {\bfseries Output:} an estimate of $\log p(\v x \vert \omega)$
   \end{algorithmic}
   \smallskip
   
    \begin{algorithmic}[1]
   \STATE Generate $\epsilon = (\epsilon_1,\dots,\epsilon_d)$ for (\ref{eq:TraceEst}). \smallskip
   \STATE {\bfseries function} \textsc{odefunc}$(z,t)$
    \STATE \hspace{1em}Compute $\tilde{\mu}(z,t)$ via (\ref{eq:ItoCorrection}).\hfill $\triangleright$  It\^{o} correction
    \STATE \hspace{1em}{\bfseries return} $\tilde{\mu}(z,t) + \sigma(z,t) \frac{\dd \tilde{B}_t(\omega)}{\dd t}$. 
   \STATE{\bfseries end function} \smallskip
   \STATE {\bfseries function} \textsc{aug}$((z,\log p_t),t)$
   \STATE \hspace{1em}$f_t \gets $ \textsc{odefunc}($z$, $t$, $\omega$)
   \STATE \hspace{1em}$J_t \gets -\nabla_z(\epsilon \cdot f_t) \cdot\epsilon$\hfill $\triangleright$  Trace estimator (\ref{eq:TraceEst}); $n = 1$
   \STATE \hspace{1em}{\bfseries return} $(f_t, J_t)$
   \STATE{\bfseries end function} \smallskip
   \STATE $(z,\Delta\log p_t^\omega) \gets $ \textsc{odesolve}(\textsc{aug},$(x,0)$,$0$,$T$)
   \STATE {\bfseries return} $\log p_0(z) - \Delta \log p_t^\omega$
\end{algorithmic}
\label{alg:DensityEst}
\end{algorithm}

A number of techniques exist for debiasing the logarithm of (\ref{eq:DensityEstNaive}) --- see \citet{rhee2015unbiased} and \citet{rischard2018unbiased}, for example. 
Alternatively, we lie in the setting of \emph{semi-implicit variational inference} seen in \citet{yin2018semi} and \citet{titsias2018unbiased}, and those techniques directly extend to our case as well. Naturally, it would be easiest to instead optimize the upper bound
\[
-\log p_t^\theta(Z_t) \leq -\mathbb{E}_{\omega} \log p_t^\theta(Z_t\vert \omega),
\]
and in many cases we have found this to be effective. Observing that
\begin{equation}
\label{eq:ELBO}
\log p_t^\theta(Z_t) - \kl(q \Vert p_{\omega \vert Z_t}^{\theta}) = \mathbb{E}_\omega \log p_t^\theta(Z_t \vert \omega),
\end{equation}
minimizing $-\mathbb{E}_{\omega}\log p_t^{\theta}(Z_t\vert \omega)$ maximizes the true log-likelihood regularized by the KL-divergence between the prior and posterior distributions for $\omega$, which reduces the effect of noise on the model. At the same time, parameterizations of the diffusion coefficient that allow $\|\sigma\|$ to shrink to zero will often do so, and should be avoided to remain distinct from a continuous normalizing flow.

\section{Numerical experiments}
\label{sec:Numerics}

\subsection{Samplers and density estimation from data}
For our first experiments, we train a stochastic normalizing flow (\ref{eq:MainSDE}) --- using Algorithm \ref{alg:DensityEst} with the upper bound (\ref{eq:ELBO}) --- to data generated from a specified target distribution. 
For our drift function, we adopt the same architecture used in the toy examples of \citet{grathwohl2018ffjord}; a four-layer fully-connected neural network with 64 hidden units in each layer. Dependence on time is removed to ensure a time-homogeneous, and hence, potentially ergodic SDE after training. All networks were trained using Adagrad \citep{duchi2011adaptive}, with $p_0 \sim \mathcal{N}(0,I)$ and a batch size of 1000 samples. 

\subsubsection{A two-dimensional toy example}
\label{sec:Toy}
In our first example, our data is generated from the banana-shaped distribution
\[
p(x,y) \propto \exp(-\tfrac12 (x^2 + \tfrac12 (x^2 + y)^2)).
\]
Two choices of diffusion coefficient are considered: the first, where $\sigma = I$, yields a neural SDE that can be trained using the techniques of \citet{li2020scalable}. For the second, we choose
\begin{equation}
\label{eq:SigmaChoice}
\sigma(x) = \lambda \left(\begin{matrix} 1 & \sigma_1(x) \\ \sigma_2(x) & 1 \end{matrix}\right),
\end{equation}
with $\lambda = 1$, and parameterize $(\sigma_1,\sigma_2)$ by a two-layered neural network with 64 hidden units. This SDE can only be trained using our method. After training, to emulate the application of these SDEs as approximate samplers, a single sample path with 10,000 steps was simulated for each model using the Euler--Maruyama method. The resulting paths are compared in Figure \ref{fig:FixedVsVariable}. From data alone, both models constructed recurrent processes. The addition of a trainable diffusion coefficient led to improved adaptation of the sampler to the underlying curvature. 
\begin{figure}[htbp]
    \centering
    \includegraphics[width=0.6\columnwidth]{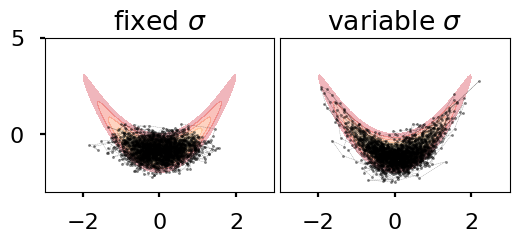}
    \caption{Sample paths from SDEs trained as stochastic normalizing flows to a banana-shaped distribution.}
    \label{fig:FixedVsVariable}
\end{figure}

\subsubsection{Visualizing regularization}
As discussed in \citet{liu2019neural}, the stochastic noise injection in SDEs is a natural form of regularization, that can potentially improve robustness to noisy or adversarial data. We visualize this effect by considering the same stochastic normalizing flows treated in \S\ref{sec:Toy} with diffusion coefficient (\ref{eq:SigmaChoice}), and adjusting the parameter $\lambda > 0$. Our data is generated in polar coordinates from a ten-pointed star-shaped distribution by
\[
\theta \sim \mathrm{Unif}(-\pi,\pi),\quad r\vert \theta \sim \mathcal{N}(\tfrac{2}{\sqrt{1 + \frac12 \sin(10 \theta)}}, \tfrac{9}{400}).
\]
In Figure \ref{fig:RegDensity}, we plot the densities for 
$\lambda \in \{0, \tfrac{1}{10}, \tfrac12, 1\}$ computed using Algorithm \ref{alg:DensityEst}, noting that the $\lambda = 0$ case corresponds to a continuous normalizing flow. Increasing $\lambda$ reveals generative models with expectedly higher variance, but with improved capacity to smooth out minor (potentially, unwanted) details.
\begin{figure}
    \centering
    \includegraphics[width=0.5\columnwidth]{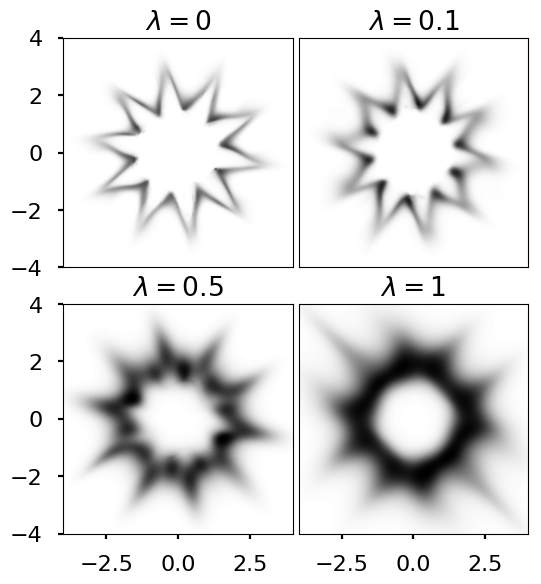}
    \caption{Density plots of stochastic normalizing flows trained to a star-shaped distribution with varying diffusion coefficients.}
    \label{fig:RegDensity}
\end{figure}

\subsection{Optimizing stochastic MCMC}
An interesting class of SDE models for approximating a target distribution $p$ are \emph{targeted diffusions}, solutions to SDEs that are $p$-ergodic. A convenient representation of such diffusions are known \citep{ma2015complete}. Because these diffusions are frequently used in MCMC algorithms, in a sense, conducting VI with respect to targeted diffusions is analogous to optimizing the convergence rate of stochastic MCMC algorithms. 

To illustrate the potential applications of stochastic normalizing flows for finding and examining optimal stochastic MCMC algorithms for a particular target distribution, we consider a basic setup, where $p$ is the one-dimensional Cauchy distribution $p(x) \propto (1 + x^2)^{-1}$. All $p$-ergodic SDEs are of the form
\begin{equation}
\label{eq:CauchySampler}
\dd Z_t = (-2 \sigma(Z_t)^2 Z_t / (1 + Z_t^2) + \tfrac12 \sigma'(Z_t)) \dd t + \sigma(Z_t) \dd B_t,
\end{equation}
and we may choose $\sigma$ arbitrarily. \emph{A priori}, an optimal choice of $\sigma$ (up to constants) to ensure rapid mixing of (\ref{eq:CauchySampler}) does not appear obvious. The present rule of thumb from second-order methods takes $\sigma \approx (\log p)''$ \citep{girolami2011riemann}. We train a stochastic normalizing flow for (\ref{eq:CauchySampler}) with $\sigma$ parameterized by a four-layer neural network with 32 hidden units in each layer. The corresponding loss function is taken to be the Kullback-Leibler divergence $\log p_T(Z_T) - \log p(Z_T)$, estimated using Algorithm \ref{alg:DensityEst}, with an $L^1$ penalty term $10^{-4}\|\v w\|_1$ over the weights $\v w$ of the neural network, to prevent taking $|\sigma| \to +\infty$. The results are presented in Figure \ref{fig:CauchyPlot}.
\begin{figure}
    \centering
    \includegraphics[width=0.7\columnwidth]{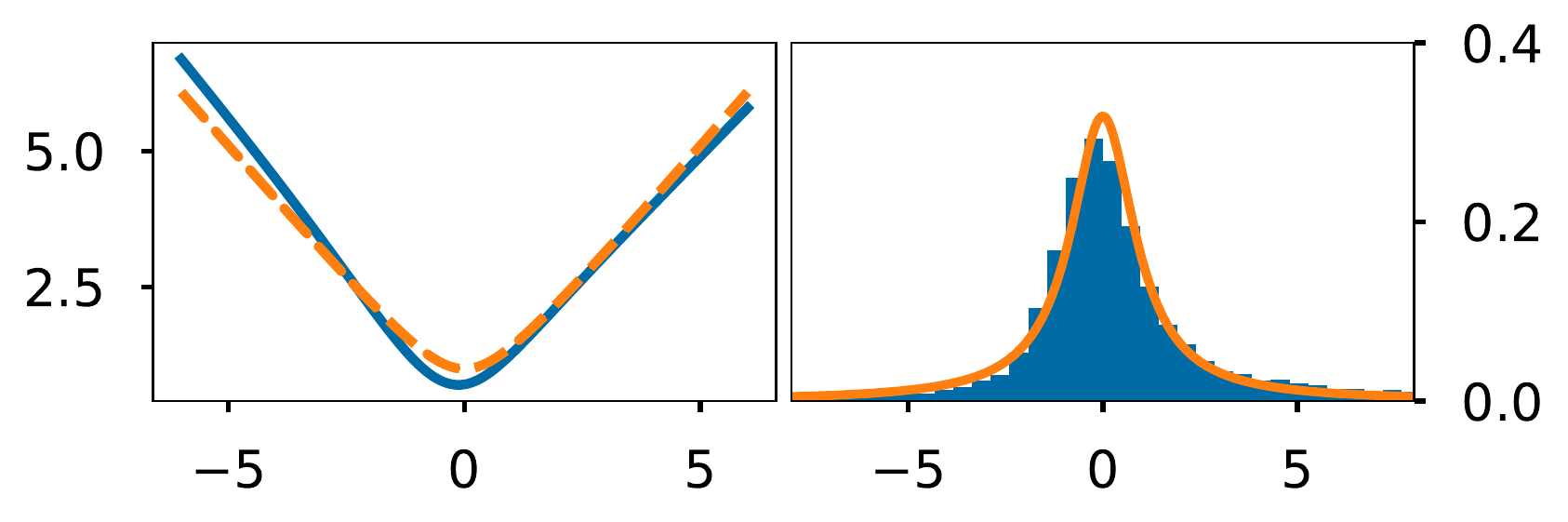}
    \caption{Stochastic normalizing flow targeting Cauchy distribution. \emph{Left}: trained diffusion coefficient (blue) compared to $\sqrt{1 + x^2}$ (orange dashed). \emph{Right}: Histogram of 20,000 generated samples (blue) and Cauchy density (orange). }
    \label{fig:CauchyPlot}
\end{figure}
Curiously, after training, we found the apparent ``optimal'' choice is approximately $\sigma(x) \propto \sqrt{1 + x^2} \propto \sqrt{\frac{(\log \varphi)'}{(\log p)'}}$, where $\varphi$ is the density of the standard normal distribution.



\section{Conclusion}
\label{sec:Conclusion}

We have extended the continuous normalizing flows framework to generative models involving SDEs. Justified by rough path theory, our framework enables practitioners of neural ODEs to apply their existing implementation for training neural SDEs. This is advantageous, as neural SDEs have been suggested to be more robust than neural ODEs in high-dimensional real-world examples \citep{liu2019neural,li2020scalable}. Stochastic normalizing flows can be implemented as a device for investigating ``optimal'' hyperparameters in stochastic MCMC, which could prove useful for informing future research, but they may require implementational improvements for high-dimensional cases, e.g., variance reduction techniques and improved loss estimators.

\paragraph{Acknowledgements.}
We would like to acknowledge DARPA, NSF, and ONR for providing partial support of this work.


\begin{thebibliography}{41}
\providecommand{\natexlab}[1]{#1}
\providecommand{\url}[1]{\texttt{#1}}
\expandafter\ifx\csname urlstyle\endcsname\relax
  \providecommand{\doi}[1]{doi: #1}\else
  \providecommand{\doi}{doi: \begingroup \urlstyle{rm}\Url}\fi

\bibitem[Beskos et~al.(2006)Beskos, Papaspiliopoulos, Roberts, and
  Fearnhead]{beskos2006exact}
Beskos, A., Papaspiliopoulos, O., Roberts, G.~O., and Fearnhead, P.
\newblock Exact and computationally efficient likelihood-based estimation for
  discretely observed diffusion processes.
\newblock \emph{Journal of the Royal Statistical Society: Series B (Statistical
  Methodology)}, 68\penalty0 (3):\penalty0 333--382, 2006.

\bibitem[Chen et~al.(2018)Chen, Rubanova, Bettencourt, and
  Duvenaud]{chen2018neural}
Chen, T.~Q., Rubanova, Y., Bettencourt, J., and Duvenaud, D.~K.
\newblock Neural ordinary differential equations.
\newblock In \emph{Advances in Neural Information Processing Systems}, pp.\
  6571--6583, 2018.

\bibitem[Davie(2008)]{davie2008differential}
Davie, A.~M.
\newblock Differential equations driven by rough paths: an approach via
  discrete approximation.
\newblock \emph{Applied Mathematics Research eXpress}, 2008, 2008.

\bibitem[Dinh et~al.(2015)Dinh, Krueger, and Bengio]{dinh2014nice}
Dinh, L., Krueger, D., and Bengio, Y.
\newblock {NICE: Non-linear independent components estimation}.
\newblock In \emph{Proceedings of the 3rd International Conference on Learning
  Representations (ICLR)}, 2015.

\bibitem[Duchi et~al.(2011)Duchi, Hazan, and Singer]{duchi2011adaptive}
Duchi, J., Hazan, E., and Singer, Y.
\newblock Adaptive subgradient methods for online learning and stochastic
  optimization.
\newblock \emph{Journal of Machine Learning Research}, 12:\penalty0 2121--2159,
  2011.

\bibitem[Dupont et~al.(2019)Dupont, Doucet, and Teh]{dupont2019augmented}
Dupont, E., Doucet, A., and Teh, Y.~W.
\newblock Augmented neural {ODEs}.
\newblock In \emph{Advances in Neural Information Processing Systems}, 2019.

\bibitem[Evans(2012)]{evans2012introduction}
Evans, L.~C.
\newblock \emph{An introduction to stochastic differential equations},
  volume~82.
\newblock American Mathematical Society, 2012.

\bibitem[Friz \& Hairer(2014)Friz and Hairer]{friz2014course}
Friz, P. and Hairer, M.
\newblock \emph{A Course on Rough Paths}.
\newblock Springer International Publishing, 2014.

\bibitem[Friz \& Victoir(2010)Friz and Victoir]{friz2010multidimensional}
Friz, P.~K. and Victoir, N.~B.
\newblock \emph{Multidimensional stochastic processes as rough paths: theory
  and applications}, volume 120.
\newblock Cambridge University Press, 2010.

\bibitem[Gholami et~al.(2019)Gholami, Keutzer, and Biros]{gholami2019anode}
Gholami, A., Keutzer, K., and Biros, G.
\newblock {ANODE: Unconditionally Accurate Memory-Efficient Gradients for
  Neural ODEs}.
\newblock In \emph{Proceedings of the Twenty-Eighth International Joint
  Conference on Artificial Intelligence (IJCAI-19)}, 2019.

\bibitem[Girolami \& Calderhead(2011)Girolami and
  Calderhead]{girolami2011riemann}
Girolami, M. and Calderhead, B.
\newblock {Riemann manifold Langevin and Hamiltonian Monte Carlo methods}.
\newblock \emph{Journal of the Royal Statistical Society: Series B (Statistical
  Methodology)}, 73\penalty0 (2):\penalty0 123--214, 2011.

\bibitem[Grathwohl et~al.(2018)Grathwohl, Chen, Betterncourt, Sutskever, and
  Duvenaud]{grathwohl2018ffjord}
Grathwohl, W., Chen, R.~T., Betterncourt, J., Sutskever, I., and Duvenaud, D.
\newblock {FFJORD: Free-form continuous dynamics for scalable reversible
  generative models}.
\newblock \emph{arXiv preprint arXiv:1810.01367}, 2018.

\bibitem[Hairer \& Wanner(1996)Hairer and Wanner]{hairer1996solving}
Hairer, E. and Wanner, G.
\newblock \emph{{Solving Ordinary Differential Equations II}}, volume~14 of
  \emph{Springer Series in Computational Mathematics}.
\newblock Springer-Verlag Berlin Heidelberg, 1996.

\bibitem[Hodgkinson et~al.(2019)Hodgkinson, Salomone, and
  Roosta]{hodgkinson2019implicit}
Hodgkinson, L., Salomone, R., and Roosta, F.
\newblock {Implicit Langevin algorithms for sampling from log-concave
  densities}.
\newblock \emph{arXiv preprint arXiv:1903.12322}, 2019.

\bibitem[Hurn et~al.(2007)Hurn, Jeisman, and Lindsay]{hurn2007seeing}
Hurn, A.~S., Jeisman, J., and Lindsay, K.~A.
\newblock {Seeing the wood for the trees: A critical evaluation of methods to
  estimate the parameters of stochastic differential equations}.
\newblock \emph{Journal of Financial Econometrics}, 5\penalty0 (3):\penalty0
  390--455, 2007.

\bibitem[Jia \& Benson(2019)Jia and Benson]{jia2019neural}
Jia, J. and Benson, A.~R.
\newblock {Neural Jump Stochastic Differential Equations}.
\newblock In \emph{Proceedings of the 33rd Conference on Neural Information
  Processing Systems (NeurIPS 2019)}, 2019.

\bibitem[Kloeden \& Platen(2013)Kloeden and Platen]{kloeden2013numerical}
Kloeden, P.~E. and Platen, E.
\newblock \emph{Numerical solution of stochastic differential equations},
  volume~23.
\newblock Springer Science \& Business Media, 2013.

\bibitem[Li et~al.(2020)Li, Wong, Chen, and Duvenaud]{li2020scalable}
Li, X., Wong, T.-K.~L., Chen, R.~T., and Duvenaud, D.
\newblock Scalable gradients for stochastic differential equations.
\newblock \emph{arXiv preprint arXiv:2001.01328}, 2020.

\bibitem[Liu \& Feng(2016)Liu and Feng]{liu2016two}
Liu, Q. and Feng, Y.
\newblock Two methods for wild variational inference.
\newblock \emph{arXiv preprint arXiv:1612.00081}, 2016.

\bibitem[Liu et~al.(2019)Liu, Si, Cao, Kumar, and Hsieh]{liu2019neural}
Liu, X., Si, S., Cao, Q., Kumar, S., and Hsieh, C.-J.
\newblock {Neural SDE: Stabilizing Neural ODE Networks with Stochastic Noise}.
\newblock \emph{arXiv preprint arXiv:1906.02355}, 2019.

\bibitem[Lu et~al.(2017)Lu, Zhong, Li, and Dong]{lu2017beyond}
Lu, Y., Zhong, A., Li, Q., and Dong, B.
\newblock {Beyond finite layer neural networks: Bridging deep architectures and
  numerical differential equations}.
\newblock \emph{arXiv preprint arXiv:1710.10121}, 2017.

\bibitem[Lyons(1998)]{lyons1998differential}
Lyons, T.~J.
\newblock Differential equations driven by rough signals.
\newblock \emph{Revista Matem{\'a}tica Iberoamericana}, 14\penalty0
  (2):\penalty0 215--310, 1998.

\bibitem[Ma et~al.(2015)Ma, Chen, and Fox]{ma2015complete}
Ma, Y.-A., Chen, T., and Fox, E.
\newblock {A complete recipe for stochastic gradient MCMC}.
\newblock In \emph{Advances in Neural Information Processing Systems}, pp.\
  2917--2925, 2015.

\bibitem[Peluchetti \& Favaro(2019)Peluchetti and
  Favaro]{peluchetti2019infinitely}
Peluchetti, S. and Favaro, S.
\newblock Infinitely deep neural networks as diffusion processes.
\newblock \emph{arXiv preprint arXiv:1905.11065}, 2019.

\bibitem[Pontryagin(2018)]{pontryagin2018mathematical}
Pontryagin, L.~S.
\newblock \emph{Mathematical theory of optimal processes}.
\newblock Routledge, 2018.

\bibitem[Rezende \& Mohamed(2015)Rezende and Mohamed]{rezende2015variational}
Rezende, D.~J. and Mohamed, S.
\newblock Variational inference with normalizing flows.
\newblock In \emph{Proceedings of the 32nd International Conference on Machine
  Learning}, 2015.

\bibitem[Rhee \& Glynn(2015)Rhee and Glynn]{rhee2015unbiased}
Rhee, C.-h. and Glynn, P.~W.
\newblock {Unbiased estimation with square root convergence for SDE models}.
\newblock \emph{Operations Research}, 63\penalty0 (5):\penalty0 1026--1043,
  2015.

\bibitem[Rischard et~al.(2018)Rischard, Jacob, and
  Pillai]{rischard2018unbiased}
Rischard, M., Jacob, P.~E., and Pillai, N.
\newblock {Unbiased estimation of log normalizing constants with applications
  to Bayesian cross-validation}.
\newblock \emph{arXiv preprint arXiv:1810.01382}, 2018.

\bibitem[Roosta \& Ascher(2015)Roosta and Ascher]{roosta2015improved}
Roosta, F. and Ascher, U.
\newblock Improved bounds on sample size for implicit matrix trace estimators.
\newblock \emph{Foundations of Computational Mathematics}, 15\penalty0
  (5):\penalty0 1187--1212, 2015.

\bibitem[Ryder et~al.(2018)Ryder, Golightly, McGough, and
  Prangle]{ryder2018black}
Ryder, T., Golightly, A., McGough, A.~S., and Prangle, D.
\newblock Black-box variational inference for stochastic differential
  equations.
\newblock \emph{arXiv preprint arXiv:1802.03335}, 2018.

\bibitem[Salimans et~al.(2015)Salimans, Kingma, and
  Welling]{salimans2015markov}
Salimans, T., Kingma, D., and Welling, M.
\newblock {Markov chain Monte Carlo and variational inference: Bridging the
  gap}.
\newblock In \emph{International Conference on Machine Learning}, pp.\
  1218--1226, 2015.

\bibitem[Titsias \& Ruiz(2019)Titsias and Ruiz]{titsias2018unbiased}
Titsias, M.~K. and Ruiz, F.~J.
\newblock Unbiased implicit variational inference.
\newblock In \emph{Proceedings of the 22nd International Conference on
  Artificial Intelligence and Statistics (AISTATS)}, 2019.

\bibitem[Twardowska(1996)]{twardowska1996wong}
Twardowska, K.
\newblock {Wong-Zakai approximations for stochastic differential equations}.
\newblock \emph{Acta Applicandae Mathematica}, 43\penalty0 (3):\penalty0
  317--359, 1996.

\bibitem[Tzen \& Raginsky(2019)Tzen and Raginsky]{tzen2019neural}
Tzen, B. and Raginsky, M.
\newblock {Neural Stochastic Differential Equations: Deep Latent Gaussian
  Models in the Diffusion Limit}.
\newblock \emph{arXiv preprint arXiv:1905.09883}, 2019.

\bibitem[van~den Berg et~al.(2018)van~den Berg, Hasenclever, Tomczak, and
  Welling]{berg2018sylvester}
van~den Berg, R., Hasenclever, L., Tomczak, J., and Welling, M.
\newblock Sylvester normalizing flows for variational inference.
\newblock In \emph{Proceedings of the Conference on Uncertainty in Artificial
  Intelligence (UAI)}, 2018.

\bibitem[Wolf et~al.(2016)Wolf, Karl, and van~der Smagt]{wolf2016variational}
Wolf, C., Karl, M., and van~der Smagt, P.
\newblock {Variational inference with Hamiltonian Monte Carlo}.
\newblock \emph{arXiv preprint arXiv:1609.08203}, 2016.

\bibitem[Wong \& Zakai(1965)Wong and Zakai]{wong1965convergence}
Wong, E. and Zakai, M.
\newblock On the convergence of ordinary integrals to stochastic integrals.
\newblock \emph{The Annals of Mathematical Statistics}, 36\penalty0
  (5):\penalty0 1560--1564, 1965.

\bibitem[Xu et~al.(2019)Xu, Quiroz, Kohn, and Sisson]{xu2019variance}
Xu, M., Quiroz, M., Kohn, R., and Sisson, S.~A.
\newblock Variance reduction properties of the reparameterization trick.
\newblock In \emph{The 22nd International Conference on Artificial Intelligence
  and Statistics}, pp.\  2711--2720, 2019.

\bibitem[Yin \& Zhou(2018)Yin and Zhou]{yin2018semi}
Yin, M. and Zhou, M.
\newblock Semi-implicit variational inference.
\newblock In \emph{Proceedings of the 35th International Conference on Machine
  Learning}, 2018.

\bibitem[Young(1936)]{young1936inequality}
Young, L.~C.
\newblock {An inequality of the H{\"o}lder type, connected with Stieltjes
  integration}.
\newblock \emph{Acta Mathematica}, 67:\penalty0 251--282, 1936.

\bibitem[Zhang et~al.(2019)Zhang, Yao, Gholami, Keutzer, Gonzalez, Biros, and
  Mahoney]{zhang2019anodev2}
Zhang, T., Yao, Z., Gholami, A., Keutzer, K., Gonzalez, J., Biros, G., and
  Mahoney, M.~W.
\newblock {ANODEV2: A Coupled Neural ODE Evolution Framework}.
\newblock \emph{Advances in Neural Information Processing Systems}, 2019.

\end{thebibliography}
\end{document}